\documentclass[letterpaper, 10 pt, conference]{ieeeconf} 
\IEEEoverridecommandlockouts
\usepackage{graphicx} 
\usepackage{xcolor}
\usepackage{verbatim}
\usepackage{float}
\usepackage{amsmath}
\usepackage{amsfonts}
\usepackage{algorithm}
\usepackage{algorithmicx}
\usepackage{subcaption}
\usepackage{multicol}
\usepackage[noend]{algpseudocode}

\newtheorem{proposition}{Proposition}
\newtheorem{theorem}{Theorem}

\author{Kaier Liang$^{1}$, Guang Yang$^{2}$, Mingyu Cai$^{3}$, Cristian-Ioan Vasile$^{1}$
\thanks{$^{1}$Kaier Liang and Cristian-Ioan Vasile  are with the Mechanical Engineering and Mechanics Department at Lehigh University, PA, USA: \{kal221, cvasile\}@lehigh.edu}
\thanks{$^{2}$Guang Yang {\tt\small gyang5052@gmail.com}}
\thanks{$^{3}$Mingyu Cai is with the Department of Mechanical Engineering at University of California Riverside, CA, USA: {\tt\small mingyu.cai@ucr.edu}}
}

\begin{document}
\title{\Large \bf Safe Navigation in Dynamic Environments Using Data-Driven Koopman Operators and Conformal Prediction}
\maketitle
\begin{abstract}
We propose a novel framework for safe navigation in dynamic environments by integrating Koopman operator theory with conformal prediction. Our approach leverages data-driven Koopman approximation to learn nonlinear dynamics and employs conformal prediction to quantify uncertainty, providing statistical guarantees on approximation errors. This uncertainty is effectively incorporated into a Model Predictive Controller (MPC) formulation through constraint tightening, ensuring robust safety guarantees. We implement a layered control architecture with a reference generator providing waypoints for safe navigation. The effectiveness of our methods is validated in simulation. 
\end{abstract}

\section{Introduction}
\label{sec:intro}
Autonomous navigation represents one of the most critical challenges in robotics research.  Advanced applications—including drone delivery, self-driving vehicles, and emergency response with autonomous system all require motion planning frameworks that can guide robots through complex environments with precision, safety, and efficiency. Despite significant advances in the field, there are two challenges: nonlinear dynamics and state uncertainty. First, many real-world robotic platforms exhibit nonlinear dynamics, such as quadrupeds, bipedal robots and soft robots. Traditional control approaches that rely on standard linearization techniques \cite{khalil2002nonlinear} or simplified models often fail to capture the full  complexity of these systems, and therefore achieve underwhelming performance. Second, any autonomous system modeling can introduce errors, often included in state uncertainty. This uncertainty manifests in multiple forms: sensor measurement noise, imperfect dynamics models, and unmodeled disturbances. Any model-based control approaches rely on an accurate model to perform well. These two challenges create barrier to developing reliable autonomous navigation systems for safety-critical applications. 

Originating from studying fluid dynamics \cite{brunton2021modern}, there are recent works that utilize Koopman Operators to control nonlinear system \cite{shi2024koopman,abraham2017model, bruder2019modeling, mamakoukas2019local}. In particular, it has been demonstrated that Koopman Operator methods allow the controller to directly handle highly nonlinear dynamical systems, such as soft robots~\cite{bruder2019modeling}, robotic fish~\cite{mamakoukas2019local} and quadcopters~\cite{10253454}. It has also been demonstrated that the Koopman operator approach outperforms standard feedback linearization techniques \cite{brunton2021modern} in many applications. The intuition is nonlinear dynamics can be fully expressed by infinite-dimensional linear operators in Hilbert spaces. As a result, a linear representation of the original nonlinear system can be obtained without losing model correctness. In practice, only a finite number of states of the representation is used to best approximate infinite-dimensional operators, and the observation function can be either hand-designed \cite{netto2021analytical} or directly learned using deep learning \cite{lusch2018deep}. As a trade off, the approximation error arrises due to the finite representation. Therefore, tracking performance could be affected when using it for model-based controllers. In addition, the design of the lifting functions is critical \cite{bruder2019modeling}, with approaches ranging from simple polynomial functions to radial basis functions. Recently, neural networks have emerged as an attractive method for learning these functions. That being said, upon deriving the Koopman operator-based model, the design of optimal controllers becomes straightforward thanks to the linear representation \cite{korda2018linear}.

Conformal Prediction provides a versatile framework applicable across machine learning domains, such as classification tasks \cite{norinder2017binary, zecchin2024generalization}, by dynamically generating prediction sets that quantify confidence levels during inference. This methodology has recently gained traction in robotics applications \cite{lindemann2023safe, yang2023safe}, where it establishes reliable confidence bounds around prediction models, effectively quantifying uncertainties related to dynamic obstacles and system states in autonomous platforms. Furthermore, \cite{chee2024uncertainty} incorporates exogenous disturbances into the dynamic system model and proposes a reference generator scheme that enables safe navigation under the assumption of close tracking performance.

However, the assumption of exchangeability of the data can be violated in many robotics applications \cite{barber2023conformal} due to distribution shift. To overcome this issue, weighted conformal prediction is introduced in \cite{tibshirani2019conformal}. In the context of motion planning, it has been shown its success in combination with model predictive controllers to achieve safe navigation with probabilistic guarantee \cite{lindemann2023safe, chee2024uncertainty}, thanks to its flexibility to adapt to any type of learned models.

The Model Predictive Control (MPC) has been widely used in robotics \cite{erez2013integrated, carron2019data, nascimento2018nonholonomic}. The downside of this approach is the requirement for an accurate model. While many robotics systems can be described by differential equations with analytical forms \cite{sa2012system}, there exist many robotics systems for which closed-form solutions are difficult to derive. Recently, data-driven MPC approaches \cite{amos2018differentiable, wang2022improved, wang2022koopman} have gained significant prominence in the field. There are also works that combine MPC and learning to achieve safe navigation \cite{yang2020model, hewing2020learning, koller2018learning}. Recently, with the integration of learned Koopman operators, the model based control approach demonstrates remarkable effectiveness for designing both Linear Quadratic Regulator controllers \cite{yang2023lqr, manaa2024koopman} and Model Predictive Controllers \cite{wang2022improved}, wherein linear representations of the nonlinear systems are derived through data-driven Koopman methods. Notably, \cite{wang2022improved} demonstrates that Koopman-based MPC successfully controls highly nonlinear soft actuators. In this work, we aim to further explore the use of data-driven MPC for controlling nonlinear systems.

In this paper, we present a unified control framework that combines Koopman operator theory with conformal prediction to manage nonlinear dynamical systems while ensuring safe navigation under system uncertainty. Our approach consists of two key phases: offline training and online control synthesis. In the offline phase, we train a Koopman-lifted system using data collected from the nonlinear system. During online control synthesis, a reference generator produces state references, which are fed into a Model Predictive Controller (MPC) operating on the trained Koopman-lifted linear system dynamics. Conformal prediction is employed to account for approximated state errors with probabilistic guarantees. To preserve the linearity of the MPC, obstacles are represented using linear constraint formulations. This framework enables the implementation of a computationally efficient linear MPC that effectively controls nonlinear systems with probabilistic safety assurances.

Our key contributions are as follows: We propose a control framework for nonlinear systems using Model Predictive Control (MPC), integrating Koopman operator theory and conformal prediction. Our approach explicitly accounts for Koopman state error and incorporates its uncertainties directly into the MPC constraints. To ensure the convexity of the MPC optimization problem, we approximate both static and dynamic obstacles using a set of linear constraints. Additionally, we introduce a hierarchical control architecture, where a reference generator guides the MPC controller, enhancing system performance and constraint satisfaction. Our method demonstrates a significant improvement in navigation performance compared to the baseline. Finally, we provide a theoretical proof of the proposed control framework, which is validated through simulation.




\section{Preliminaries and Problem Formulation}
\label{sec:preliminary}
In this section, we introduce notations and provide background on the key frameworks utilized in our approach and formulate the problem statement.
\subsection{Dynamical System}
Consider a discrete-time dynamical system
\begin{equation}\label{eq:dynamicSystem}
\begin{aligned}
    x_{k+1} &= f(x_k, u_k), \\
    x_0 &= x_{o}    
\end{aligned}
\end{equation}
with $x_k \in \mathcal{X} \subset \mathbb{R}^n$, $u_k \in \mathcal{U} \subset \mathbb{R}^m$, where $\mathcal{X}$ and $\mathcal{U}$ are  permissible state and control sets, respectively. The $x_{o} \in \mathcal{X}$ is the initial state. The function $f: \mathbb{R}^{n} \times \mathbb{R}^{m} \rightarrow \mathbb{R}^{n}$ denotes the dynamics of the system.

\subsection{Koopman Operator}
Given the dynamical system \eqref{eq:dynamicSystem}, let $g(x_k) \in \mathcal{G}$ be an observation function that maps the state $x_k$ to an observable $y_k$:
\[
y_k = g(x_k)
\]
The \( y_k \) is called an observable that measures the state in the Hilbert space of the original dynamical system. The Koopman operator $\mathcal{K}$ governs the evolution of the observable over time. The application of the Koopman operator $\mathcal{K}$ on the observable $g(x_k)$ can expressed as
\[
(\mathcal{K} g)(x) = g(x) \circ f(x),
\]
where $\circ$ denotes the composition operator.
Note that for any functions $g_1,g_2: \mathcal{X} \to \mathcal{G} \text{ and all } a_1,a_2 \in\mathbb{R}$, we have

\begin{align*}
\mathcal{K}(a_1g_1 + a_2g_2) &= (a_1g_1 + a_2g_2)\circ f \\
&= a_1g_1 \circ f + a_2g_2 \circ f \\
&= a_1 \mathcal{K}(g_1) + a_2 \mathcal{K}(g_2).
\end{align*}

\subsection{Conformal Prediction}
Conformal prediction is a statistical framework for constructing prediction regions with provable  guarantees\cite{vovk2005algorithmic}. 
Given a dataset $\mathcal{D}=\left\{\left(z_i, y_i\right)\right\}_{i=1}^n$, where $z_i$ represents input features and $y_i$ represents outputs, and a prediction model $\hat{f}$, the nonconformity score for a sample is defined as $s_i=$ $d\left(y_i, \hat{f}\left(z_i\right)\right)$, where $d$ is a suitable distance measure. For a new input $z_{\text {new }}$ and a user-defined failure probability $\alpha \in(0,1)$, the prediction region is constructed as $\mathcal{C}_\alpha\left(z_{\text {new }}\right)=\{y \mid \left.d\left(y, \hat{f}\left(z_{\text {new }}\right)\right) \leq Q_{1-\alpha}\right\}$, where $Q_{1-\alpha}$ is the $(1-\alpha)$-th quantile of the calibration set of nonconformity scores. This guarantees $\mathbb{P}\left(y_{\text {new }} \in \mathcal{C}_\alpha\left(z_{\text {new }}\right)\right) \geq 1-\alpha$.

Conformal Prediction differs fundamentally from Bayesian methods by requiring minimal distributional assumptions, making it especially valuable when the underlying data distribution is unknown. Rather than relying on specific parametric models, CP's theoretical guarantees stem from a single assumption: data exchangeability. This property ensures that the joint distribution of the data remains invariant under any permutation of sample indices.

\subsection{Problem Formulation}

Given a data set $\mathcal{D} = \{(x_k, u_k, x_{k+1})_{i}\}_{i=1}^M$ collected from a dynamical system, our goal is to navigate safely from an initial state to a target state while maintaining safety constraints.

Let $\mathcal{F} \subset \mathbb{R}^n$ denote a safe set in state space. For the initial and target states $x_{o}, x_{target} \in \mathcal{F}$, our objective is to synthesize a control sequence $\{u_k\}_{k=0}^{N-1}$ that guides the system to the destination within a small tolerance $\delta$ while ensuring that all intermediate states remain within the safe set:

\begin{equation}\label{eq:mpc}
\begin{aligned}
    &x_{k+1} = f(x_k, u_k), \\
    &x_0 = x_{o}\\
    &\|x_N - x_{target}\| \leq \delta\\
    &\quad \text{s.t. }x_{k} \in \mathcal{F}, \forall k \in \{0,1,\ldots,N\},
\end{aligned}
\end{equation}

where $f(x_k, u_k)$ is the unknown dynamics.

\section{Solution}
Since the true dynamics $f$ are unknown, we leverage the data set $\mathcal{D}$ to learn a Koopman operator representation of the nonlinear dynamics, quantify uncertainty in the learned model using conformal prediction, and design a safe control strategy with statistical guarantees. This data-driven formulation enables safe navigation even when the underlying system dynamics are not explicitly known but can be observed through collected data. The overall structure is given in Alg.~\ref{alg:koopman_rg_mpc}, where we first learn the Koopman representation and quantify uncertainty during the offline phase, followed by an online control phase that leverages a reference generator (RG) and model predictive control with tightened constraints obtained from conformal prediction. In the following sections, we detail each component of this framework.

\subsection{Uncertainty quantification for Koopman Operator}
In general, the Extended Dynamic Mode Decomposition (EDMD) \cite{ouala2023extending} enables the selection of modes explicitly. In this work, we are interested in learning Koopman operator using data-driven approach, enabling a linear representation without handcrafting the lifting function. More specifically, we obtain a control-affine approximation of the Koopman dynamics in the lifted space expressed as
\begin{equation}
\label{eq:lifted_system}
    z_{k+1} = A z_k + B u_k
\end{equation}
where $A \in \mathbb{R}^{p \times p}$ and $B \in \mathbb{R}^{p \times m}$ are matrices approximating the Koopman operator, $p$ is the dimension of the lifted space in $z$, where $z_k = g(x_k)$. The lifting function $g$ is expressed as an encoder using neural network. 
The training process incorporates reconstruction loss, linear dynamics loss, multi-step prediction loss and infinity norm loss as in~\cite{lusch2018deep,wang2022data}.

In theory, an infinite-dimensional Koopman operator representation can perfectly model a nonlinear dynamical system. However, in practice, finite-dimensional approximations in the lifted space~\eqref{eq:lifted_system} will introduce approximation errors when mapping between the original system dynamics and their lifted counterparts. 
We denote the error as $e = \|x - \hat{x}\|$, which quantifies the discrepancy between the true state $x$ and the state $\hat{x}$ reconstructed from the lifted representation, where $\hat{x} = g^{-1}(z)$ is the predicted state obtained by applying the inverse mapping from the Koopman space back to the original state space.
To address this uncertainty, we employ conformal prediction to establish statistically rigorous upper bounds on these approximation errors. 

Conformal prediction requires that the data used to compute for the prediction intervals and for calibration are exchangeable. 
To ensure that the exchangeability assumption holds in our framework, we collect calibration data using control inputs that mimic the structure and characteristics such as reference tracking that will be used during deployment using the trained Koopman model in line 3 of Alg.~\ref{alg:koopman_rg_mpc}. 

Specifically, our calibration data set $\mathcal{D}_{\text{cal}} = \{(x_i, \hat{x}_i)\}_{i=1}^n$ consists of pairs of true states $x_i$ and their Koopman predicted counterparts $\hat{x}_i$, gathered from trajectories that reflect typical system operation. This approach ensures that the error distribution encountered during actual navigation closely matches that observed during calibration. Therefore, the non-conformity scores $s_i = \|x_i - \hat{x}_i\|$ computed from these calibration data provide valid prediction intervals when applied to new states during online operation. 

Given a new state $x_k$ at time step $k$, we compute its prediction $\hat{x}_{k+1} = g^{-1}(A g(x_k) + B u_k)$ using the Koopman operator matrices $A$ and $B$. To establish prediction intervals with confidence level $1-\alpha$, we compute:
$$Q_{1-\alpha}(S) = \text{Quantile}_{1-\alpha}(\{s_1, s_2, ..., s_n\})$$
where $\text{Quantile}_{1-\alpha}$ represents the $(1-\alpha)$-th quantile of the empirical distribution of nonconformity scores. This yields the uncertainty interval:
$$\mathcal{C}(\hat{x}_k) = \{x \in \mathbb{R}^n : \|x - \hat{x}\| \leq Q_{1-\alpha}(S)\}$$
This approach ensures that the actual state $x_{k}$ will be contained within $\mathcal{C}(\hat{x}_k)$ with probability at least $1-\alpha$, provided that the nonconformity scores are exchangeable random variables.
The exchangeability condition can be alleviated by using weighted conformal prediction (WCP)\cite{barber2023conformal}, which computes a weighted quantile:
$$Q_{1-\alpha}(S) = \text{Quantile}_{1-\alpha}\left(\sum_{i=1}^n w_i \delta_{s_i} + w_{n+1}\delta_{\infty}\right)$$
where $\{w_i\}_{i=1}^{n+1}$ are normalized weights and $\delta_a$ denotes a point mass at $a$. However, WCP introduces additional computational burden and requires careful weight selection, potentially resulting in wider prediction intervals.

\begin{proposition}
\label{prop:cp}
Consider a Koopman dynamics in lifted space in~\eqref{eq:lifted_system}, for any control $u\in\mathcal{U}$. Let $\{(x_i, \hat{x}_i)\}_{i=1}^{n}$ be a calibration dataset where $x_i$ represents the true state and $\hat{x}_i$ the state estimated using the Koopman model with stochastic sensor measurements. Assume the nonconformity scores $s_i = \|x_i - \hat{x}_i\|$ are exchangeable random variables. For a new state estimation pair $(x_{n+1}, \hat{x}_{n+1})$ and a failure probability $\alpha \in (0,1)$, the conformal prediction region $\mathcal{C}_{1-\alpha}$ satisfies:
\begin{equation}
\mathbb{P}(\|x_{n+1} - \hat{x}_{n+1}\| \leq Q_{1-\alpha}) \geq 1-\alpha
\end{equation}
where $Q_{1-\alpha} = \text{Quantile}_{1-\alpha}\{s_1, \ldots, s_n, \infty\}$ is the $(1-\alpha)$-th quantile of the empirical distribution of nonconformity scores.
\end{proposition}

\begin{proof}
Under the exchangeability assumption of the nonconformity scores, the rank of $e_{n+1}$ among $\{e_1,...,e_n,e_{n+1}\}$ is uniformly distributed over $\{1,...,n+1\}$. By construction, $Q_{1-\alpha}$ is the $\lceil(n+1)(1-\alpha)\rceil$-th smallest value in $\{e_1,...,e_n,\infty\}$. Therefore, $\mathbb{P}(e_{n+1} \leq Q_{1-\alpha}) \geq \lceil(n+1)(1-\alpha)\rceil/(n+1) \geq 1-\alpha$.
\end{proof}

\begin{theorem}
\label{th: cp}
    For a given $\mathcal{F} = \{h(x) > 0\}$, where the safety constraints $h(x)$ is Lipschitz continuous with constant $L$, and quantile  threshold $\Delta$. Define the tighten set $\mathcal{F}_{\Delta} = \{x \in \mathbb{R}^n : h(x) \geq \Delta\}$. Then, if the predicted state from Koopman dynamics $\hat{x} \in \mathcal{F}_\Delta$, the actual system state $x$ remains in $\mathcal{F}$ with probability at least $1-\alpha$:
\begin{equation}
\mathbb{P}(x \in \mathcal{F} \mid \hat{x} \in \mathcal{F}_\Delta) \geq 1-\alpha
\end{equation}
\end{theorem}

\begin{proof}
    For any $x, \hat{x}$, $|h(x) - h(\hat x)| \leq L\|x - \hat{x}\|$. If $\hat{x} \in \mathcal{F}_{\Delta}$, then $h(\hat{x}) \geq \Delta = LQ_{1-\alpha}$.
    
    For any $\|x - \hat{x}\| \leq Q_{1-\alpha}$, we have
    \begin{align*}
        h(x) &\geq h(\hat{x}) - L\|x - \hat{x}\| \\
        &\geq LQ_{1-\alpha} - LQ_{1-\alpha} = 0
    \end{align*}
To ensure the strict safety $h(x) >0$, redefine $\Delta = LQ_{1-\alpha} + \epsilon$ for $\epsilon > 0$. Then
\begin{align*}
    h(x) \geq \epsilon
\end{align*} 
which implies $x \in \mathcal{F}$ and $\mathbb{P}(x \in \mathcal{F} \mid \hat{x} \in \mathcal{F}_\Delta) \geq 1-\alpha$.
\end{proof}

\begin{algorithm}
    \caption{Koopman-based MPC with  Conformal Prediction}
    \label{alg:koopman_rg_mpc}
    \begin{algorithmic}[1]
        \State \textbf{Input:} Dataset $\mathcal{D} = \{(x_k, u_k, x_{k+1})_{i}\}_{i=1}^M$, initial state $x_0$, goal state $x_{\text{goal}}$, confidence level $1-\alpha$, tolerance $\delta$
        
        \Statex {\color{blue}\textit{// Offline Training Phase}}
        
        \State Learn Koopman operator matrices $(A,B)$ and lifting function $g(\cdot)$ from dataset $\mathcal{D}$
        
        \State Collect calibration dataset $\mathcal{D}_{\text{cal}}$ under deployment-like conditions

        \State Compute nonconformity scores and determine threshold $Q_{1-\alpha}$ as the $(1-\alpha)$-quantile of $\{\|x_i - \hat{x}_i\|\}_{i=1}^n$ from $\mathcal{D}_{\text{cal}}$
        
        \State Compute constraint tightening margin $\Delta$
        
        \Statex {\color{blue}\textit{// Online Control Phase}}
        
        \While{$\| x_k - x_{\text{goal}}\| > \delta$}
        
            \State Measure current state $x_k$, compute lifted state $z_k = g(x_k)$
            
            \State Update obstacle positions and constraints.
            
            \State Compute safe reference waypoint:
            $$x_{\text{ref}} = RG(x_k, x_{\text{goal}},\Delta)$$
            
            \State Solve MPC optimization \eqref{eq: MPC} with reference $x_{\text{ref}}$ and tightened constraints in Koopman space
            
            \State Apply first control input to real system from MPC solution
            
            \State Set $k = k + 1$
            
        \EndWhile
        
        \State {\bf Output:} Safe trajectory $\{x_0,x_1,\dots,x_k\}$ 
        
    \end{algorithmic}
\end{algorithm}
From Theorem~\ref{th: cp}, we can quantify the approximation error bounds for our Koopman operator representation learned. This enables us to derive tightened safety constraints in the Koopman space that provide probabilistic guarantees on constraint satisfaction in the original state space. These tightened constraints are then directly used for control synthesis to ensure safe navigation despite approximation errors.
\subsection{Model Predictive Controller}
In this work, we opt to use MPC in Koopman space for controlling the nonlinear dynamical systems. At each time step $k$, the MPC is formulated as the following optimization problem:

\begin{subequations}
\label{eq: MPC}
    \begin{align}
J = \min_{z, u} \quad & \sum_{i=0}^{N-1} \|(z_{k+i|k} - \bar z_{k})\|_{Q}^2 + \|u_{k+i|k}\|_{R}^2\label{eq:obj}\\
\text{subject to} \quad & z_{k+i+1|k} = A z_{k+i|k} + B u_{k+i|k}, \\
&\quad h(g^{-1}(z_{k+i|k})) \geq \Delta \label{eq:cp_constrain}\\\notag
& u_{k+i|k} \in \mathcal{U}, \quad i = 0, 1, \ldots, N-1,
\end{align}
\end{subequations}
where $N$ is the prediction horizon. The matrices $Q \in \mathbb{R}^{p\times p}$ and $R\in\mathbb{R}^{m \times m}$ are penalties for tracking error and control, respectively. $\hat {z}_k$ is the tracking reference at time step $k$. The safe set is defined in \eqref{eq:cp_constrain} from the conformal prediction.

The reference state $\bar z_{k}$ in the objective function \eqref{eq:obj} is provided by a high-level reference generator, which aims to guide the system toward the goal state while the MPC ensures safety. 
In practice, the parameter $\Delta$ balances safety and performance. A higher value of $\Delta$ (higher confidence level $1-\alpha$) provides stronger safety guarantees but leads to more conservative behavior.

In this paper, we linearize obstacle constraints in \eqref{eq:cp_constrain} using half space planes shown in Figure~\ref{fig:saferegion}, following the convex feasible set approach formalized in~\cite{liu2018convex}. This linearization approach translate the safety constraints into standard linear constraints
$Hx + b \geq 0$,
where $H\in\mathbb{R}^{n\times n}$ is a matrix representing the normal vectors of the half space constraints, and $b\in\mathbb{R}^{N}$ is an offset vector. This half space formulation is effective for representing convex obstacles because it ensures that the region outside the obstacle is convex, allowing for efficient optimization techniques to be applied. 
Combining these linear constraints with the linear representation of system dynamics provided by the Koopman operator, we obtain a quadratic optimization problem in~\eqref{eq: MPC}. 
\begin{figure}
    \centering
    \includegraphics[width=0.5\linewidth]{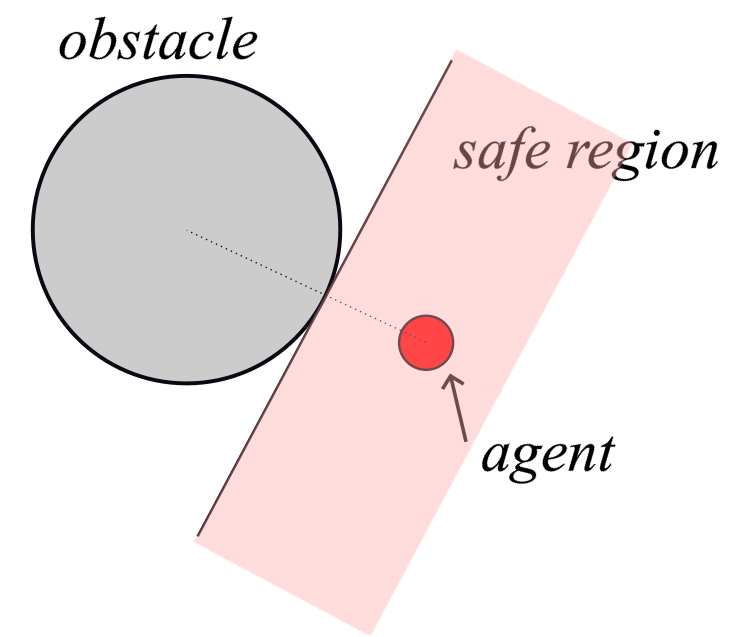}
    \caption{Half Space Constraints}
    \label{fig:saferegion}
\end{figure}

Implementing direct hard safety constraints $h(g^{-1}(z_k)) \geq \Delta$ in MPC can lead to infeasibility problems when dealing with system uncertainty. When the nominal predicted state $\hat{x}_k$ from MPC is near the constraint boundary, the actual system state $x_k$ may deviate due to model uncertainty while still remaining physically safe (i.e., $\Delta \geq h(x_k) \geq 0$. However, this deviation can cause the optimization problem to become infeasible in subsequent time steps, as the hard constraint would be violated in the prediction horizon despite the system remaining in the safe set. 

Established robust MPC approaches for handling such uncertainties in different approaches such as constraint softening and tube MPC~\cite{langson2004robust}, which employs additional feedback control layers to guarantee convergence despite bounded uncertainties. In this paper, we use the constraint softening idea introduced in \cite{zeilinger2014soft}, where the MPC introduces slack variables with associated penalties rather than enforcing hard constraints. 
Specifically, we reformulate the safety constraints in (6c) as:
\begin{align}
Hx + b &\geq \Delta - \epsilon_s - \epsilon_k, \quad \forall k \in \{0,\ldots,N-1\} \\
\epsilon_s + \epsilon_k &\leq \epsilon_{\text{max}}, \quad \epsilon_s, \epsilon_k \geq 0
\end{align}
and add penalty terms to the objective function in~\eqref{eq:obj}:
\begin{align}
J_{penalty} = \sum_{k=0}^{N-1} \|\epsilon_k\|_S^2 + \rho_1\|\epsilon_k\|_p + \|\epsilon_s\|_S^2 + \rho_1\|\epsilon_s\|_p
\end{align}
where $\epsilon_s$ represents the minimum constraint relaxation necessary to maintain feasibility, and $\epsilon_k$ represents additional relaxation at each prediction step. $S$ is a positive definite weight matrix, $\rho_1 > 0$ is a penalty parameter, and $p \in \{1,\infty\}$. This slack variable structure can allow constraint violations when necessary.

Exclusive dependence on soft constraints for path planning also presents challenges in practice. Since penalty terms for slack variables only become active after constraints are violated, the controller may allow the system to approach boundaries too closely before taking corrective action. This reactive behavior can result in inefficient zigzag trajectories and degraded performance. Our approach addresses this limitation by incorporating a high-level reference generator that generates local tracking targets $\bar z_k$ while proactively considering tightened constraints, leading to smoother and more efficient navigation.

The reference generator in line 9 can be implemented through various approaches, including optimization-based and sampling-based strategies. In this paper, we provide only the minimum environment information and the agent's approximated dynamics (i.e., single integrator dynamics). Different from \cite{chee2024uncertainty}, we do not requires a close tracking assumption or strict safety in the reference generator as these are already encoded in the MPC, which ensures that safety guarantees are maintained at the lowest level of control hierarchy. The primary purpose of the reference generator is to provide directional guidance that proactively accounts for tightened constraints, thereby preventing the zigzagging behavior that typically occurs when relying solely on reactive soft constraints. Our results demonstrate that a simple reference generator is sufficient without requiring precise agent control.

In lines 10-12, we solve the optimization problem \eqref{eq: MPC} to obtain the optimal control sequence, apply only the first control input to the real system, and then measure the resulting next state. The process repeats until the task is finished.

In summary, our framework addresses the approximated state error from the learned Koopman model through conformal prediction techniques. By incorporating linear safety constraints and linear representations of the non-linear dynamics, we construct a linear MPC  that can be computed efficiently and has a probabilistic guarantee.


\section{Experiment}
We validate our framework on a unicycle model with nonlinear discrete-time dynamics:
\begin{equation}
\begin{aligned}
        x_{k+1} &= x_k + dt\cdot v_k\cos(\theta_k)\\
        y_{k+1} &= y_k + dt\cdot v_k\sin(\theta_k)\\
        \theta_{k+1} &= \theta_k + dt\cdot \omega_k,
\end{aligned}
\end{equation}
where $[x_k, y_k]$ represents the position and $\theta_k$ is the heading angle at time step $k$. The control inputs are the translational velocity $v_k$ and angular velocity $\omega_k$, and $dt = 0.1s$ is the sampling time. The dimension of the lifted space is set to be $11$, and learn the Koopman representation of these dynamics using collected trajectory data. 
The task is to navigate through obstacles and visit a series of target points in sequence with the prediction horizon $N=10$.

\begin{figure}[hbt!]
    \centering
\includegraphics[width=1.2\linewidth]{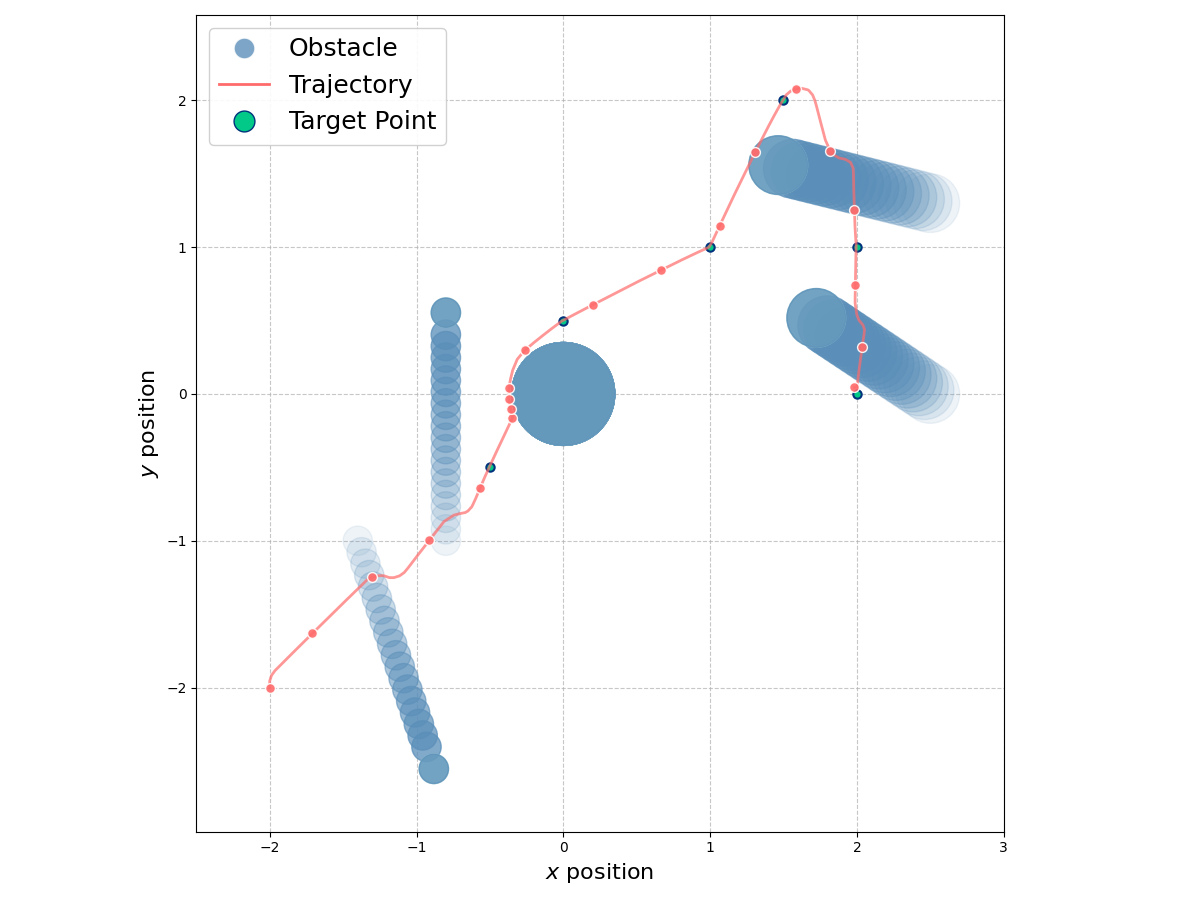}
    \caption{Navigation with Dynamic Environment: Agent's trajectory is shown in orange line, where it goes through a sequence of designated target points (marked as green dots). Dynamic obstacles are shown as blue shaded regions where the opacity represents their position over time. Darker blue indicates the obstacle's more recent position, while lighter blue shows where it was previously. Agent starts at location $(-2, -2)$ and ends at location $(2, 0)$.}
    \label{fig:env}
\end{figure}

Figure~\ref{fig:env} shows the trajectory in navigating an environment with obstacles. The linear MPC problem formulated in~\eqref{eq:mpc} is solved using Gurobi~\cite{gurobi}. The half space safety constraint is defined as $a_ix + b_iy + c_i \geq C_{1-\alpha}$, where $a_i, b_i, c_i$ are calculated based on the position of the i-th obstacle and the position of the agent $x, y$. We choose $\alpha := 0.02$ meaning $98$ percent confidence for the confidence interval. The trajectory show the our approach successfully reaches the targets and maintain safety.

\begin{figure}[hbt!]
    \centering
    \includegraphics[width=1\linewidth]{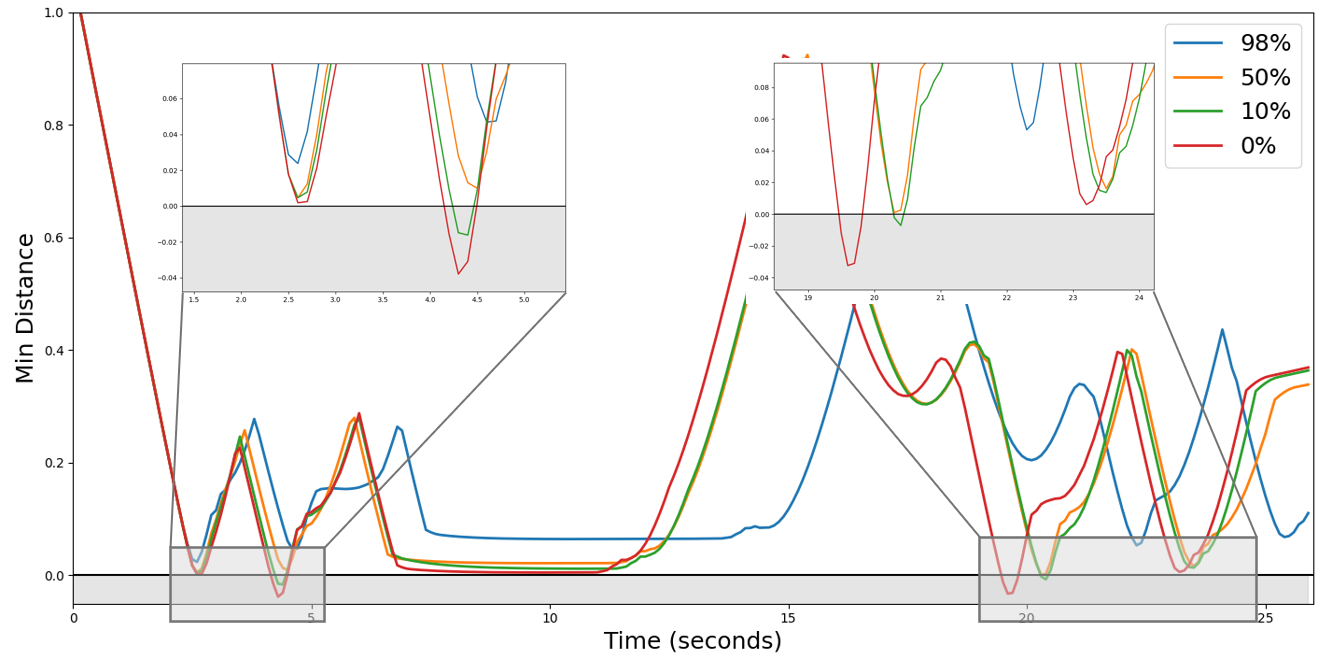}
    \caption{Comparison between different confidence levels. Collision time intervals are highlighted. Lower confidence levels lead to increased risk, resulting more collisions. Any lines that dips into the grey area represent collisions. }
    \label{fig:compare_cp}
\end{figure}

Figure~\ref{fig:compare_cp} presents a comparison of navigation performance using different levels of conformal prediction confidence $98\%, 50\%, 10\%$ and $0$. The y-axis represents the minimum distance to the nearest obstacle throughout the trajectory, with positive values indicating safe clearance and negative values representing collision with obstacles. As the confidence level decreases (and consequently the safety margin $\Delta$ shrinks), we observe increased risk-taking behavior in the navigation. The results clearly demonstrate that lower confidence thresholds $10\%$ and $0$ lead to safety violations at multiple time points, while the highest confidence level $98\%$ consistently maintains a safer distance from obstacles throughout the simulation.
\begin{figure}[htb!]
    \centering
    \begin{subfigure}[b]{0.9\linewidth}
        \includegraphics[width=\linewidth]{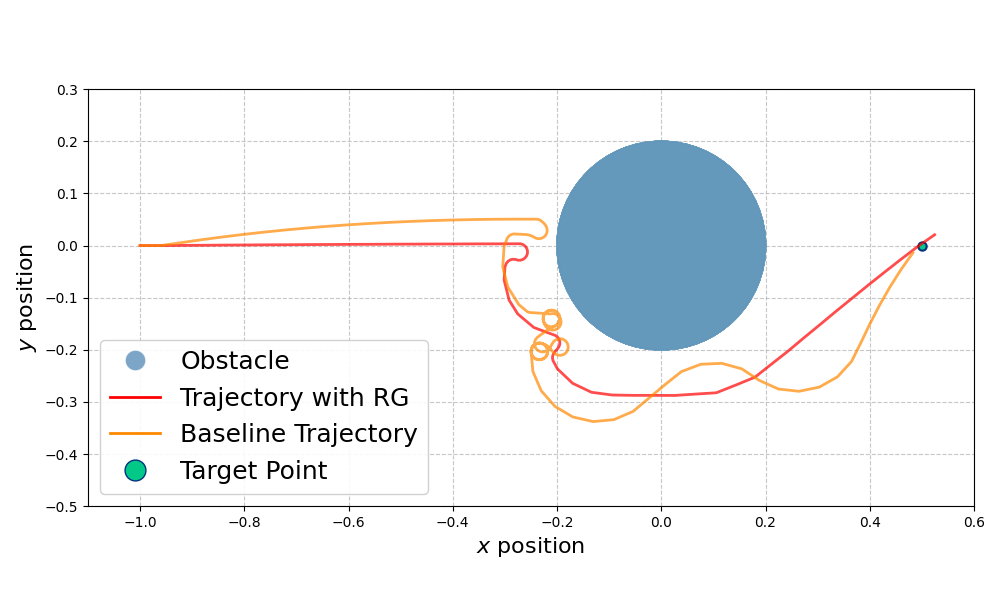}
        \caption{} 
        \label{subfig: rg_traj}
    \end{subfigure}
    \begin{subfigure}[b]{1\linewidth}
    
    \includegraphics[width=\linewidth]
    {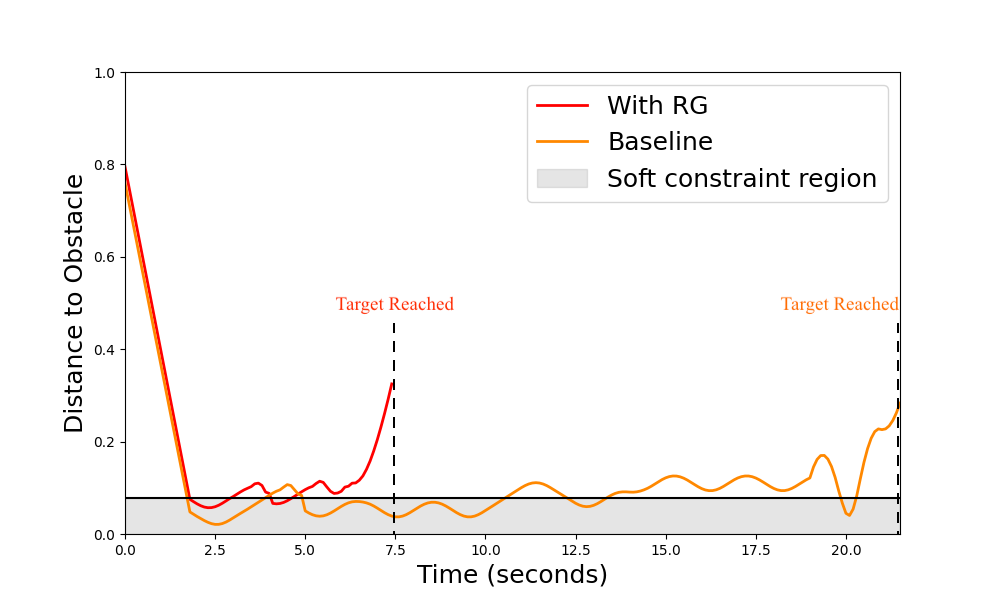} 
    \caption{} 
    \label{subfig: rg_dist}
    \end{subfigure}
    \caption{Comparison between the proposed approach with reference generator and baseline using only soft constraints. Agent starts from $(-1, 0)$ and ends at $(0.5,0)$
    (a) The trajectory comparison reveals that our agent follows a smooth, efficient path when guided by the reference generator, in contrast to the zigzag movement exhibited by the baseline method. (b) Distance to obstacles during navigation, where the shaded region indicates soft constraint activation without collision. Our method also outperforms the baseline in terms of task completion time.}
    \label{fig: compare_rg}
\end{figure}

Figure \ref{fig: compare_rg} illustrates the performance difference between our proposed approach with reference generator and a baseline using only soft constraints, both configured with identical conformal prediction confidence levels. The task shown in Figure~\ref{subfig: rg_traj} is to navigate through one obstacle and reach the target. As evident in the trajectories, our reference generator approach proactively generates optimized local tracking points that lead to smoother trajectories and improved navigation performance with fewer near-obstacle encounters. Figure~\ref{subfig: rg_dist} shows the distance to the obstacle at each time before reaching the target. The baseline method displays distinctive zigzag patterns as it relies solely on soft constraints that activate only after violations occur. This reactive approach not only forces the system into constant compensatory corrections but also significantly delays reaching to targets, resulting in less efficient overall performance. It is worthwhile to point out that unlike \cite{chee2024uncertainty}, our approach uses Koopman operator theory to derive linear representations of system dynamics directly, eliminating the need for a nominal model with separate uncertainty characterization. By integrating Koopman operators with conformal prediction, we establish a unified framework that simultaneously addresses nonlinear dynamics and uncertainty coming from the approximation error. Additionally, our implementation of soft constraints enhances practicality by allowing temporary violations, thus significantly alleviating the infeasibility issues commonly associated with hard constraints. Our approach enables a more convenient practical implementation during robot deployments thanks to the linear dynamics and obstacle constraints. 

For comparison, our linear MPC framework exhibits a improved computational efficiency compared with nonlinear MPC. Under a test scenario where all MPCs are feasible throughout the mission, the computation time per MPC step was reduced for scenarios with linearized constraints that maintain convexity of the MPC problem, demonstrating a 26\% speed improvement compared to scenarios with nonlinear constraints, which introduce non-convexity. Moreover, the performance gain remains mostly the same across conformal prediction confidence levels. The results demonstrate the importance of formulate linear MPC by learning the linear representation of the nonlinear systems, as well as incorporating the obstacles as linear constraints. 

\section{Conclusion}
This paper introduces a novel framework for safe navigation by combining Koopman operator theory with conformal prediction. By utilizing the Koopman operator to achieve a linear representation of nonlinear systems and applying conformal prediction to quantify uncertainty, we establish statistical guarantees on safety constraints. The problem is formulated within a linear MPC framework, and the effectiveness of our approach is demonstrated through experiments. The results also highlight the capability of a reference governor in steering the agent away from obstacles while maintaining navigation efficiency. Future work will focus on extending this approach to robots with more complex dynamics and deploying it in real-world scenarios. Additionally, we will explore more effective methods for modeling obstacles online by representing them as convex sets in dynamic environments






\bibliographystyle{IEEEtran}
\bibliography{reference}

\end{document}